\documentclass[english,11pt]{article}
\usepackage{amsmath,amsfonts,amssymb,amsthm,fancyhdr,bm,verbatim, hyperref, color}
\usepackage{fullpage}
\usepackage{mathtools}
\usepackage{todonotes}
\usepackage{float}
\usepackage{mdframed}
\makeatletter
\numberwithin{equation}{section}
\numberwithin{figure}{section}
\theoremstyle{plain}
\newtheorem{thm}{\protect\theoremname}
\theoremstyle{plain}
\newtheorem{cor}[thm]{\protect\corollaryname}
\theoremstyle{plain}
\newtheorem{lem}[thm]{\protect\lemmaname}
\newtheorem{claim}[thm]{\protect\claimname}

\newtheorem{defn}[thm]{\protect\definitionname}
\numberwithin{thm}{section}

\newcommand{\grad}{\nabla}

\makeatother

\usepackage{babel}
\providecommand{\corollaryname}{Corollary}
\providecommand{\claimname}{Claim}
\providecommand{\questionname}{Question}

\providecommand{\lemmaname}{Lemma}
\providecommand{\theoremname}{Theorem}
\providecommand{\conjecturename}{Conjecture}
\providecommand{\propositionname}{Proposition}
\providecommand{\definitionname}{Definition}
\newcommand{\R}{\mathbb{R}}

\newcommand{\ut}{u_{\theta}}
\newcommand{\gut}{\grad \ut}
\newcommand{\gu}{\grad u}

\definecolor{mygreen}{RGB}{80,180,0}

\title{Spectral Clustering on Large Datasets: When Does it Work?\\Theory from Continuous Clustering and Density Cheeger-Buser}

\author{
Timothy Chu \\
  Google\\
  \texttt{timothyzchu@gmail.com}\\
\and
Gary L.~Miller\\
  CMU\\
  \texttt{glmiller@cs.cmu.edu}\\
\and
Noel J.~Walkington\\
  CMU\\
  \texttt{noelw@andrew.cmu.edu}\\
}

\begin{document}

\maketitle
\thispagestyle{empty} 
\setcounter{page}{0}

\begin{abstract}

Spectral clustering is one of the most popular clustering algorithms that has stood the test of time. It is simple to describe, can be implemented using standard linear algebra, and often finds better clusters than traditional clustering algorithms like $k$-means and $k$-centers.  The foundational algorithm for two-way spectral clustering, by Shi and Malik, creates a geometric graph from data and finds a spectral cut of the graph.

In modern machine learning, many data sets are modeled as a large number of points drawn from a probability density function. Little is known about when spectral clustering works in this setting -- and when it doesn't.  Past researchers justified spectral clustering by appealing to the graph Cheeger inequality (which states that the spectral cut of a graph approximates the ``Normalized Cut''), but this justification is known to break down on large data sets. 

We provide theoretically-informed intuition about spectral clustering on large data sets drawn from probability densities, by proving when a continuous form of spectral clustering considered by past researchers (the \textit{unweighted spectral cut}) finds good clusters of the underlying density itself. Our work suggests that Shi-Malik spectral clustering works well on data drawn from mixtures of Laplace distributions, and works poorly on data drawn from certain other densities, such as a density we call the `square-root trough'. 


Our core theorem proves that \textit{weighted} spectral cuts have low weighted isoperimetry for all probability densities. Our key tool is a new Cheeger-Buser inequality for all probability densities, including discontinuous ones.

\end{abstract}
\newpage 
\thispagestyle{empty}
\setcounter{page}{0}
\begingroup
  \hypersetup{hidelinks} 
  \tableofcontents
\endgroup
\newpage
\section{Introduction}\label{sec:intro}

Spectral clustering~\cite{ShiMalik97, NgSpectral01} is one of the most popular methods in machine learning for finding clusters in data~\cite{von2007tutorial, von2008consistency, belkin2004semisup, belkin2007convergence, belkin2005towards, lkll11spectral, yhj09spectral, kl19spectral, ccl21spectral, zr18spectral, acss20, k19spectral, ms22spectral, SpectralNet, bga20spectral}, and has inspired modern deep-learning based approaches to clustering and semantic segmentation~\cite{law17deepSpectral, SpectralNet, melas22deepSpectral}. The foundational paper in spectral clustering, by Shi and Malik~\cite{ShiMalik97}, gave a two-way normalized spectral clustering algorithm based on spectral cuts of a geometric graph from data. 

Despite this popularity, little is known theoretically about the quality of clusters given by spectral clustering on large data sets.  Researchers previously justified spectral clustering~\cite{belkin2004semisup,ms22spectral} by appealing to the graph Cheeger inequality~\cite{AlonM84}, which guarantees that the spectral cut of a geometric graph has a low surface area (sum of edge weights crossing the cut) to volume (sum of edge weights on the smaller side of the cut) ratio.~\footnote{This ratio is known as the \textit{isoperimetry} of the cut, and minimizing this quantity is known as the ``Normalized Cut"~\cite{ShiMalik97, von2007tutorial} or ``Sparsest Cut"~\cite{ms90sparse, lr99sparse, chawla05sparse} problem.} However, graph Cheeger gives very poor guarantees on large or infinite graphs~\cite{t21, bhlmy15buser, kkrt16buser}. \footnote{The Cheeger inequality guarantees that spectral cuts on a graph are a $1/\Phi$ approximation to optimal (where $\Phi$ is the minimum isoperimetry over all cuts). On geometric graphs like the $n$ point line or the $\sqrt{n} \times \sqrt{n}$ grid, $1/\Phi$ diverges since the surface area scales more slowly with $n$ than the volume. For large or infinite graphs, researchers needed graph \textit{Buser} inequalities for tighter spectral cut guarantees~\cite{t21, bhlmy15buser, kkrt16buser}. 
}

One fundamental assumption of machine learning is that data can be modeled as a points in Euclidean space drawn from an underlying probability density. Researchers have analyzed spectral clustering by examining its behavior as the number of samples drawn from a density grows large~\cite{von2008consistency, TrillosVariational15, TrillosConsis16}. Past work showed that normalized Shi-Malik spectral clustering~\cite{ShiMalik97} on a large sample from probability density $\rho$ is related to its continuous form: the \textbf{unweighted spectral cut} of $\rho^2$~\cite{TrillosVariational15}. 

\textbf{Our Contribution:} 

 \begin{mdframed}
\begin{figure}[H]
  \includegraphics[width=0.52\linewidth]{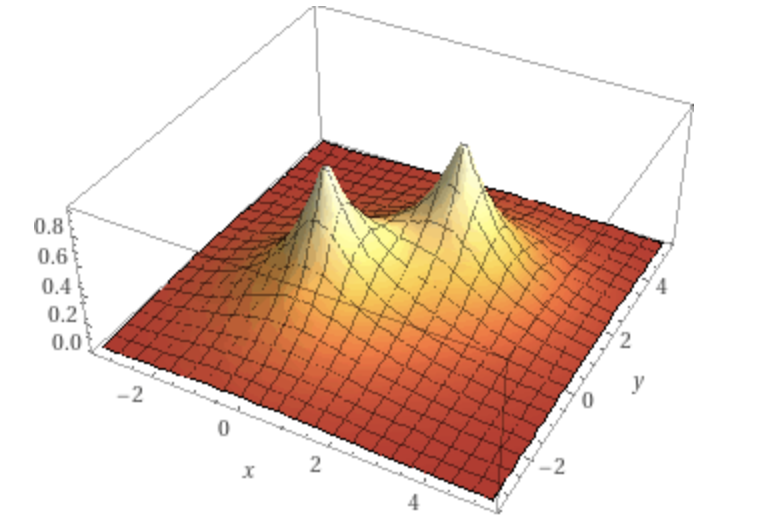}
  \caption{The mixture of Laplace distributions $e^{-\sqrt{x^2+y^2}} + e^{-\sqrt{(x-2)^2 + (y-2)^2}}$ scaled to have integral one. We prove continuous (normalized Shi-Malik) spectral clustering finds a cut with low surface area to volume ratio on
  mixtures of Laplace distributions. This gives intuition that Shi-Malik spectral clustering may find good clusters on data drawn from these densities.}
  \label{fig:laplace}
\end{figure}
\end{mdframed}

 \begin{mdframed}
\begin{figure}[H]
  \includegraphics[width=0.55\linewidth]{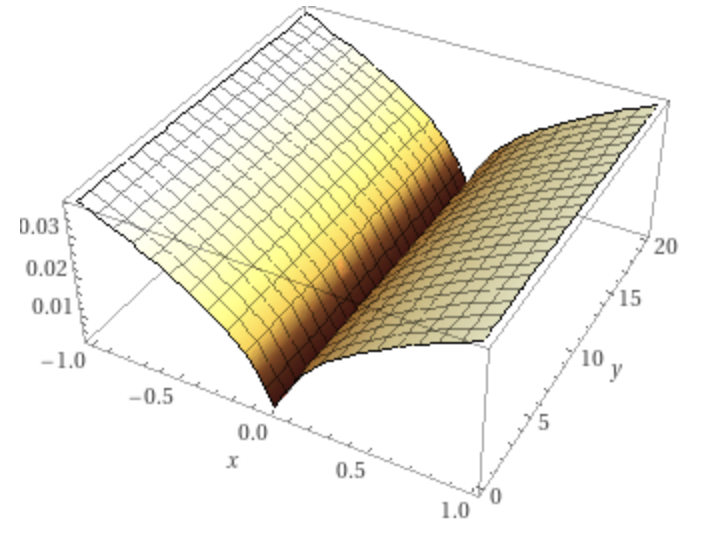}
  \caption{The `square-root trough' density $\sqrt{|x| + 1/2^{20}}$ on $[-1, 1] \times [0, 20]$, scaled to have integral one. We show the continuous (normalized Shi-Malik) spectral clustering finds the suboptimal cut $y=10$, when the cut $x=0$ is much smaller. This gives intuition that Shi-Malik spectral clustering may find bad clusters on data drawn from these densities.}
  \label{fig:sqrt-trough}
\end{figure}
\end{mdframed}

We prove that continuous (normalized Shi-Malik) spectral clustering works well on mixtures of Laplace distributions. This, combined with~\cite{TrillosVariational15}, gives theoretically informed intuition that normalized Shi-Malik spectral clustering works well on large data sets drawn from these densities -- but it stops short of formally proving that Shi-Malik spectral clustering on a sample drawn from a density converges to an good cut on these densities as the number of samples grows large. 
We also show that continuous (normalized Shi-Malik) spectral clustering works poorly on densities similar to the `square-root trough' density in Figure~\ref{fig:sqrt-trough}.

We prove the unweighted spectral cut on any function $f$ has a good surface area to volume ratio (also known as \textit{isoperimetry}), as long as the probability density does not vary by more than a constant factor on balls of radius $1$. Examples of such functions include mixtures of Laplace distributions. 
 Loosely speaking, a cut with good isoperimetry will cut through regions of low probability mass, while dividing the space into two regions of comparatively high probability mass. Some kind of condition on the function is necessary: in Appendix~\ref{app:counterexample}, we show that unweighted spectral cuts have very bad isoperimetry on the function $f(x, y) = x + \frac{1}{2^K}$ supported on $[-1, 1] \times [0, K]$, where $K$ is large.

Our core theorem proves that weighted spectral cuts have low weighted isoperimetry for all probability densities. This theorem implies the aforementioned results on unweighted spectral cuts. The weights are functions from each point $x$ in the domain to the positive reals, chosen based on how large of a ball can be drawn centered at the point $x$ so that the density at each point in the ball is within a constant factor of the density at any other point in the ball.  Our paper gives the first spectral cuts with any isoperimetry guarantee, weighted or otherwise, on all densities: even discontinuous ones.


Our core tool is a new Cheeger~\cite{Cheeger70} and Buser~\cite{Buser82} inequality on probability densities, which relates weighted isoperimetry of a probability density to its weighted fundamental eigenvalue. Weights are necessary: unweighted Buser inequalities are false for the function $\rho(x) = |x| + \epsilon$ on the interval $[-1, 1]$ (see Section 6 of~\cite{cmww20}). Our paper gives the first Buser-type inequality on all probability densities, including discontinuous ones. Past Buser inequalities on densities relied on log-concavity~\cite{KLS95, Lee18Survey}, Lipschitzness~\cite{cmww20}, or a bound on some analog of Ricci curvature for densities (one analog is based on the Hessian of the log density, see Theorem 1.4 in~\cite{Milman09}).  Buser inequalities were previously known for manifolds~\cite{Buser82, ledoux2004spectral, Milman09} and graphs~\cite{kkrt16buser, bhlmy15buser} with some notion of bounded curvature.  

The core difficulty in our work is proving the density Buser inequality: the density Cheeger inequality is straightforward from existing proofs of the Cheeger inequality in manifold~\cite{Cheeger70} and graph~\cite{AlonM84} settings.  
Using weights to ensure a correct Cheeger-Buser inequality is necessary even in the case of graphs~\cite{ChungBook97}\footnote{Cheeger inequalities using normalized Laplacians set the weights of each vertex to be its weighted degree.}, and should not be considered a major drawback of our approach. 
\subsection{Road Map}

\indent For the remainder of Section~\ref{sec:intro}, we provide definitions and give our main theorem: Theorem~\ref{thm:weighted-spectral-cut}, which states weighted spectral cuts have good weighted isoperimetry. We also give our core tool, Theorem~\ref{thm:cheeger-buser}, a weighted Cheeger-Buser inequality on densities. In Section~\ref{sec:cor}, we explain how our weighted spectral cut theorem proves our result on unweighted spectral cuts. We survey past work in Section~\ref{sec:past-work}. 

In Section~\ref{sec:1d}, we give a simple proof of the density Buser inequality in Theorem~\ref{thm:cheeger-buser}, in one dimension. In Section~\ref{sec:technique}, we overview proof techniques for the $d$-dimensional density Buser inequality, which requires considerably heavier computation than our one dimensional proof. We give future directions in Section~\ref{sec:conclusion}. 

Appendix~\ref{app:counterexample} shows that unweighted spectral cuts can have poor unweighted isoperimetry. Appendix~\ref{app:spectral} shows how to prove Theorem~\ref{thm:weighted-spectral-cut} from Theorem~\ref{thm:cheeger-buser}.  
 Appendix~\ref{app:cheeger} proves the Cheeger inequality in Theorem~\ref{thm:cheeger-buser}. Appendix~\ref{app:buser-start} proves the Buser inequality in Theorem~\ref{thm:cheeger-buser}. In Appendix~\ref{app:more-cors}, we present special cases of Theorem~\ref{thm:cheeger-buser}.

\subsection{Definitions} \label{sec:definitions}
\begin{defn}
Let $m, c : \R^d \rightarrow \R$. The \textbf{$(m,c)$-isoperimetry} of a set $A\subset \R^d$ is defined as 
\[ \Phi_{m, c}(A) := \frac{\int_{B(A)} c(x) dx}{\min \left(\int_{A} m(x)dx, \int_{\R^d - A} m(x)dx\right)} \]

where $B(A)$ is the boundary of $A$. 
\end{defn}
The \textbf{unweighted isoperimetry} of set $A$ with respect to $\rho$ is defined as $\Phi_{\rho, \rho}(A)$.\footnote{The unweighted isoperimetry of $A$ corresponds to the $\rho$-weighted surface area to volume ratio for the set $A$ or $\R^d \backslash A$, depending which has smaller $\rho$-weighted mass.}
\begin{defn}
Let $m: \R^d \rightarrow \R$ and $c: \R^d \rightarrow \R$ be integrable functions.  Define the \textbf{$(m, c)$ isoperimetric constant $\Phi_{m, c}$} as:

\[ \Phi_{m, c} := \inf_{A \subset \R^d} \Phi_{m, c}(A) \]
\end{defn}
The \textbf{unweighted isoperimetric constant} on probability density function $\rho$ is defined as $\Phi_{\rho, \rho}$

\begin{defn}
Let $m: \R^d \rightarrow \R$ and $s : \R^d \rightarrow \R$.

Define the \textbf{$(m, s)$ Rayleigh quotient of $u: \R \rightarrow \R$} to be: 
\[Q_{u, m, s} = \frac{\int_{\R^d} s(x) |\grad u(x)|^2 dx}{\int_{\R^d} m(x) |u(x)|^2 dx}
\]

Define the \textbf{$(m, s)$ fundamental eigenvalue $\lambda_2^{m, s}$} as:
\[\lambda_2^{m, s} :=\inf_{\int  m(x)u(x) dx = 0} Q_{u, m, s}
\]

Define the \textbf{$(m, s)$ fundamental eigenfunction $u_2^{m, s}$} be:
\[u_2^{m, s} :={\arg\min}_{\int  m(x)u(x) dx = 0} Q_{u, m, s}
\]
\end{defn}
This definition is a modeled off the definition of a fundamental eigenvalue of a graph Laplacian with edge and mass weights. The \textbf{unweighted fundamental eigenvalue} of probability density function $\rho$, is defined as $\lambda_2^{\rho, \rho}$, and the \textbf{unweighted fundamental eigenfunction} is $u_2^{\rho, \rho}$

\begin{defn}
Let $m, s, c: \R^d \rightarrow \R$. The \textbf{$(m,s,c)$ weighted spectral cut} is the set $S_{t'}$ defined as follows: first, find the $(m, s)$ fundamental eigenfunction $u_2^{m, s}$. Next, define $S_t := \{x \in \R^d | u_2^{m, s}(x) \geq t\}$. Finally, set 
\[ t' := \arg \min_t \Phi_{m,c}(S_t) \]
\end{defn}
Note that the weighted spectral cut is defined as a subset of $\R^d$, similar to how a cut on a graph can be defined as a subset of the graph's vertices.
\begin{defn} The \textbf{unweighted spectral cut} of a probability density $\rho$ is defined as the $(\rho, \rho, \rho)$ weighted spectral cut. 
\end{defn}
\subsection{Main Theorems: Weighted Spectral Cuts and Density Cheeger-Buser}\label{sec:theorems}
The core theorem of our paper is Theorem~\ref{thm:weighted-spectral-cut}. Our core tool is Theorem~\ref{thm:cheeger-buser}. Proving Theorem~\ref{thm:weighted-spectral-cut} using our core tool is straightforward (Appendix~\ref{app:spectral}). The core technical contribution is the Buser inequality in Theorem~\ref{thm:cheeger-buser}.
 \begin{thm}\label{thm:weighted-spectral-cut} (Weighted Spectral Cut) 
 Let $\rho:\R^d \rightarrow \R^{\geq 0}$, let $C$ be a constant, and let $R: \R^d \rightarrow \R^{\geq 0}$ be an $L$-Lipschitz function such that $\frac{\rho(y)}{\rho(z)} \leq C$ for all $y, z$ in a ball of radius $R(x)$ centered at $x$.  Let $A$ be the $(\rho, R^2\rho, R\rho)$ weighted spectral cut of $\rho$. Then:
\[ \Phi_{\rho, R\rho} \leq \Phi_{\rho, R\rho}(A) \leq O_{C}\left(\sqrt{d(L+1)\Phi_{\rho, R\rho}} + \sqrt{d} \cdot \Phi_{\rho, R \rho}\right) \]
\end{thm}


This theorem shows that weighted spectral cuts have low weighted isoperimetry. This is the core theorem in our paper. We prove Theorem~\ref{thm:weighted-spectral-cut} using a new, weighted Cheeger/ Buser inequality:

\begin{thm} \label{thm:cheeger-buser}(Weighted Cheeger/Buser Inequality for Probability Densities)

Let $\rho : \R^d \rightarrow \R^{\geq 0}$ with $\int_{x \in \R^d} \rho(x) < \infty$. For any $R : \R^d \rightarrow \R$, 

\[ \frac{\Phi_{\rho, R\rho}^2}{4} \leq \lambda_{2}^{\rho, R^2\rho} \]

Suppose furthermore that $R$ has the property that $\frac{1}{2}\rho(x) \leq \rho(y) \leq 2\rho(x)$ for any $y$ in a ball of radius $R(x)$ around $x$, for some $L$-$Lipschitz$ function $R$. Then: \[ \lambda_2^{\rho, R^2\rho} \leq d \cdot O\left((L+1)  \Phi_{\rho, R\rho} + \Phi_{\rho, R\rho}^2\right) \] 
\end{thm}
We call the first inequality the \textbf{weighted density Cheeger inequality}, and the second inequality the \textbf{weighted density Buser inequality}.  Proving Theorem~\ref{thm:weighted-spectral-cut} from Theorem~\ref{thm:cheeger-buser} is standard (see Appendix~\ref{app:spectral}).  The proof for weighted density Cheeger is straightforward (see Appendix~\ref{app:cheeger}). The core difficulty is proving the Buser inequality of Theorem~\ref{thm:cheeger-buser}, which is in Appendix~\ref{app:buser-start}. A simplified proof of our Buser inequality in one dimension is given in Section~\ref{sec:1d}.

For Theorems~\ref{thm:weighted-spectral-cut} and~\ref{thm:cheeger-buser}, the restriction that $R$ is $L$-Lipschitz is not as restrictive as it seems: for any density $\rho$, setting $R(x)$ to be the \textit{maximum} possible value so that $\rho$ doesn't vary by more than a constant factor on balls of radius $R(x)$ centered at $x$, results in $R$ being $1$-Lipschitz. Thus, we can state corollaries that do not depend on $L$:

 \begin{cor}\label{cor:weighted-spectral-cut} (Weighted Spectral Cut, no Lipschitz Condition) 
 Let $\rho:\R^d \rightarrow \R^{\geq 0}$, and let $R: \R^d \rightarrow \R$ be defined such that for all $x$, $R(x)$ is the maximum value such that $\frac{\rho(y)}{\rho(z)} \leq 2$ for all $y, z$ in a ball of radius $R(x)$ centered at $x$.  Let $A$ be the $(\rho, R^2\rho, R\rho)$ weighted spectral cut of $\rho$. Then:
\[ \Phi_{\rho, R\rho} \leq \Phi_{\rho, R\rho}(A) \leq O\left(\sqrt{d\Phi_{\rho, R\rho}} + \sqrt{d} \cdot \Phi_{\rho, R \rho}\right) \]
\end{cor}
\begin{cor} \label{cor:cheeger-buser}(Weighted Buser, no Lipschitz Condition) Let $\rho : \R^d \rightarrow \R^{\geq 0}$ with $\int_{x \in \R^d} \rho(x) < \infty$. Let $R: \R^d \rightarrow \R$ be defined such that for all $x$, $R(x)$ is the maximum value such that $\frac{\rho(y)}{\rho(z)} \leq 2$ for all $y, z$ in a ball of radius $R(x)$ centered at $x$. Then  \[ \lambda_2^{\rho, R^2\rho} \leq d \cdot O(\Phi_{\rho, R\rho} + \Phi_{\rho, R\rho}^2) \] 
\end{cor}
\begin{proof} (of Corollaries~\ref{cor:weighted-spectral-cut} and~\ref{cor:cheeger-buser}) Observe that any ball of radius $R(x)$ centered at $x$ cannot be contained strictly in the interior of a ball of radius $R(y)$ centered at $y$, by definition of $R$. Therefore, $R$ is $1$-Lipschitz, and our corollaries follow from Theorems~\ref{thm:weighted-spectral-cut} and~\ref{thm:cheeger-buser}.
\end{proof}

\subsubsection{Theorems for Continuous Spectral Clustering} \label{sec:cor}
We state and prove our main result on unweighted spectral cuts of $\rho$ (equivalently, continuous Shi-Malik spectral clustering on $\sqrt{\rho}$). 

 \begin{cor}\label{thm:unweighted-spectral-cut} (Continuous Spectral Clustering / Unweighted Spectral Cut) 
 Let $\rho:\R^d \rightarrow \R^{\geq 0}$, which does not vary by more than a constant multiple $C$ on all balls of radius $1$. Let $A$ be the unweighted spectral cut of $\rho$. Then:
\[ \Phi_{\rho, \rho} \leq \Phi_{\rho, \rho}(A) \leq O_{C,d}(\sqrt{\Phi_{\rho, \rho}} + \Phi_{\rho,  \rho}) \]
\end{cor}
\begin{proof} This follows directly from Theorem~\ref{thm:weighted-spectral-cut} when $R$ is constant. \end{proof}
If $\rho$ has the property described in the corollary, so does $\rho^2$.  Thus, the $(\rho^2, \rho^2, \rho^2)$ spectral cut has good $(\rho^2, \rho^2)$ isoperimetry when $\rho$ doesn't vary by more than a constant factor on balls of constant radius. Since $(\rho^2, \rho^2, \rho^2)$ spectral cuts are a continuous form of spectral clustering on $\rho$~\cite{TrillosVariational15} (see past work in Section~\ref{sec:past-work-spectral-clustering}), this corollary informs us about when continuous spectral clustering on $\rho$ has a good $(\rho^2, \rho^2)$ isoperimetry. We now state our unweighted Cheeger-Buser inequality on probability densities.
\begin{cor} \label{cor:log-lip} (Unweighted Cheeger-Buser)\label{cor:unweighted} If $\rho : \R^d \rightarrow \R$ does not vary by more than a constant multiple $C$ on balls of constant radius $K$, then
\[
\Omega(\Phi_{\rho, \rho}^2) \leq \lambda_2^{\rho, \rho} \leq O_{C,K}(d\Phi_{\rho, \rho}  + \Phi_{\rho, \rho}^2) \]
\end{cor}
\begin{proof} This follows directly from Theorem~\ref{thm:cheeger-buser} when $R$ is constant.\end{proof}

\subsection{Past Work and Contributions}\label{sec:past-work}

\subsubsection{Spectral Clustering on Large Datasets and Continuous Clustering}\label{sec:past-work-spectral-clustering}

Shi-Malik normalized spectral clustering~\cite{ShiMalik97} builds a geometric graph with data as vertices (using either the Gaussian kernel as edge weights, or by connecting the graph locally), and then performing a spectral sweep cut method based on Cheeger cut methods in spectral graph theory~\cite{ShiMalik97, NgSpectral01, von2007tutorial, AlonM84, ChungBook97}. Researchers have studied the behavior of spectral clustering on large data sets drawn i.i.d from a probability density~\cite{von2008consistency, TrillosVariational15, TrillosConsis16}, and have shown convergence properties of spectral clustering as the data set grows large.

Previous work established robust connections between widely used two-way spectral clustering algorithms (like the one in~\cite{ShiMalik97}), and spectral cuts applied to probability density~\cite{von2008consistency, TrillosVariational15, TrillosConsis16}. Loosely speaking,~\cite{TrillosVariational15} shows that normalized Shi-Malik spectral clustering\footnote{Normalized spectral clustering uses a normalized Laplacian of an intermediate geometric graph, while unnormalized spectral clustering uses an unnormalized Laplacian. See~\cite{von2007tutorial} for more information.} on datasets drawn i.i.d from a probability density, behaves similarly to $(\rho^2, \rho^2, \rho^2)$ weighted spectral cuts on the density when the number of samples grows large. They also show that unnormalized spectral clustering in the limit behaves similarly to $(\rho, \rho^2, \rho^2)$ spectral cuts~\footnote{These proofs currently rely on some initial assumptions about $\rho$ - for more details, refer to~\cite{TrillosVariational15}}. Corollary~\ref{cor:log-lip} show notions of isoperimetry guaranteed by normalized spectral cut methods. Our work provides theoretical justifications for normalized Shi-Malik spectral clustering on large datasets, for certain classes of densities.


 %
 
 Researchers have previously suggested that Shi-Malik clustering works due to the graph Cheeger inequality~\cite{belkin2004semisup, ms22spectral}.  As shown in the introduction, this argument breaks down for large datasets: the quality of approximation given by the graph Cheeger and Buser inequalities gets arbitrarily bad as the number of points gets large~\cite{t21}, which is undesirable in the large data setting.   On some densities, spectral cuts can give cuts with undesirable isoperimetry that defy machine learning intuition. See Appendix~\ref{app:counterexample} for more details. Other work on spectral clustering has shown its limitations on certain data sets: however, this past work focuses more on the limits of the ``normalized cut" approach that spectral clustering is based on~\cite{ng06spectrallimits}, whereas our work accepts that the normalized cut is a reasonable approach. 

\subsubsection{Cheeger and Buser Inequalities}
Cheeger and Buser inequalities have been foundational to spectral graph theory~\cite{ChungBook97, AlonM84}, manifold theory~\cite{Buser82, ledoux2004spectral, Milman09}, and probability densities~\cite{Lee18Survey, KLS95}. Our work significantly expands the scope of Buser inequalities for densities.

\textbf{For Probability Densities:} Past researchers have considered Buser-like inequalities on restricted probability densities, using unweighted isoperimetry and unweighted fundamental eigenvalues.    Log concave densities allow for a tighter bound of $\Phi^2 = \Theta_d(\lambda_2)$~\cite{KLS95, Lee18Survey}, where $d$ is the dimension of the density's domain. Log concave densities are relevant to volume estimation and sampling algorithms for convex bodies~\cite{Dyer89, Lee18Survey}, since the indicator function for any convex body is log-concave~\cite{Lee18Survey}. This line of work is related to the KLS conjecture -- a fundamental conjecture in convex geometry~\cite{KLS95, Lee18, y20, jlv22}. 
While the bounds for log concave densities are tighter than the bounds we obtain, these densities are very restricted. For example, log concave densities cannot have two local maxima. While this is not an issue for convex body sampling, past results on this line are less applicable to spectral clustering.

The work of~\cite{cmww20} shows a weighted Buser inequality relating $(\rho, \rho^2)$ isoperimetry with the $(\rho, \rho^3)$ fundamental eigenvalue for Lipschitz densities. This work suffers from a core drawback that both quantities are $0$ for densities like the Gaussian or Laplace distribution.  The spectral cut algorithm from their work, when applied to \textit{any} mixture of two Gaussians in one dimension, results in cuts arbitrarily far out on the tail. This is because the $(\rho, \rho^2)$ isoperimetry of cuts far out on the tail goes to 0.  Our theorems overcome this barrier: it can be shown that the $(\rho, R \rho)$ isoperimetry of a Gaussian mixture is non-zero, when $R: \R^d \rightarrow \R^{\geq 0}$ is chosen so that $\rho$ does not vary by more than a constant multiple on balls of radius $R(x)$ centered at $x$, and $R(x)$ is as large as possible.

The manifold Buser inequality (Theorem 1.4 in~\cite{Milman09}, first shown in~\cite{ledoux2004spectral}) addresses the case of probability densities on manifolds, but is limited to density functions whose Hessian of the log density is bounded.  In contrast, our work applies to all densities, including discontinuous ones.
\\[1mm]
 \indent \textbf{For Graphs:} The Cheeger inequality for graphs~\cite{AlonM84} states that up to constant factors, the fundamental eigenvalue of the normalized Laplacian is bounded below by the square of the graph isoperimetry\footnote{Isoperimetry in graphs is also known as sparsity or the normalized edge expansion}, and bounded above by the graph isoperimetry. This inequality is the foundation of spectral graph theory~\cite{AlonM84, ChungBook97}. Insights from these inequalities have been applied to algorithmic tasks including multi-way graph clustering~\cite{Louis12, LeeMultiway14}, Laplacian system solving~\cite{SpielmanTeng2004, KMP, ckmpprx14}, balanced cuts~\cite{osv12, cglnps20}, expander decomposition~\cite{SpielmanTeng2004, sw19},  maximum flow~\cite{CKMST}, sparsest cut~\cite{kllgt13, AlonM84, SpielmanTeng2004}, and more~\cite{GuMi95, Orecchia08, Orecchia2011, kw16, Kwok2016, Schild18, klt22cheeger}. For a deeper exposition on Cheeger's inequality and its proof, see~\cite{AlonM84, ChungBook97, SpielmanTeng2004}.

It is known that the upper bound of the Cheeger inequality graphs loses meaning as the number of vertices grows large~\cite{t21}. Researchers have established tighter upper bounds on the fundamental eigenvalue via $\textit{Buser}$ inequalities for large or infinite graphs, which require a bound on some discrete analog of Ricci curvature~\cite{t21, bhlmy15buser, kkrt16buser}. 
\\[1mm]
\indent \textbf{Cheeger and Buser Inequality for Manifolds:} The Cheeger~\cite{Cheeger70} and Buser~\cite{Buser82} inequalities on manifolds are foundational. They were first discovered in manifolds, and only later were they applied to graph theory~\cite{AlonM84} and other settings. The Cheeger inequality applies to all compact manifolds, while the Buser inequality only applies to compact manifolds with Ricci curvature bounded from below.

 Manifold Cheeger inequalities and Buser inequalities are useful to understand heat diffusion and mixing rates of random walks on manifolds~\cite{Buser82, ledoux2004spectral}. For more on Cheeger and Buser inequalities on manifolds, see~\cite{ledoux2004spectral}.

\section{A simple proof of density Buser in one dimension} \label{sec:1d}
We give a simple proof of Theorem~\ref{thm:cheeger-buser} in one dimension using the Hardy-Muckenhoupt inequality~\cite{muckenhoupt1972hardy, Schild18, MillerHardy18}. The Hardy-Muckenhoupt Inequality can be viewed as a bound on the fundamental eigenvalue of a mass-weighted, edge-weighted graph Laplacian, from a three-way split of the graph. See~\cite{MillerHardy18, Schild18} for details.

This section illustrates \textit{why} we choose our weights in Theorem~\ref{thm:cheeger-buser} the way that we do. Our choice of weights leads to a simple proof for Theorem~\ref{thm:cheeger-buser} in one dimension. 
 We conjecture (but do not prove) that the localization lemma~\cite{KLS95, LV17, jlv22} can generalize this proof to all higher dimensions.  
This paper's full proof of Theorem~\ref{thm:cheeger-buser} in higher dimension uses a different set of tools, and requires heavier computation compared to the simple proof in this section. 
\subsection{Preliminaries: Hardy Splits and Inequality}

Consider a one dimensional mass function $m: \R \rightarrow \R^{\geq}$, cut function $c: \R \rightarrow \R^{\geq 0}$, and spring function $s: \R \rightarrow R^{\geq 0}$, where $\int_{-\infty}^\infty m(x) dx < \infty$. Define a \textbf{Hardy Split} to be two reals $(a, b)$ with $a \leq b$. We define the \textbf{$(m, s)$ Hardy Value} of the Hardy Split $(a, b)$ to be:

\[
\frac{1 / \int_a^b \left(dx/s(x) \right)}{\min \left( \int_{-\infty}^a m(x) dx, \int_b^\infty m(x) dx \right) }
\]

The work of~\cite{Schild18, MillerHardy18} showed that this Hardy Value can be viewed in terms of effective    and mass weights on graphs. Loosely speaking, the numerator of the Hardy Value can be seen as the effective conductance between the sets $\{x \leq a\}$ and $\{x \geq b\}$ where $s$ is the infinitesimal conductance, and the denominator can be seen as the smaller of the masses of these two sets where $m$ is the infinitesimal mass.

\begin{mdframed}
\begin{figure}[H]
  \includegraphics[width=\linewidth]{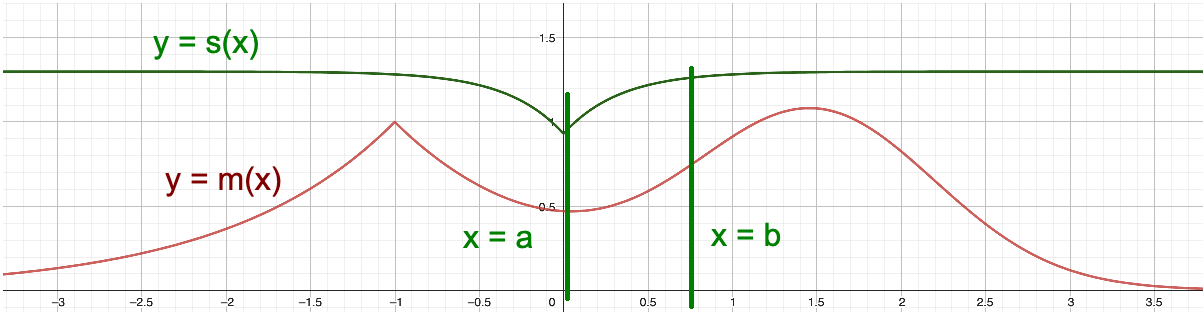}
  \caption{A Hardy split of $m$ defined by $(a, b)$. The $(m, s)$ Hardy Value of this split can be interpreted as the effective conductance between $x=a$ and $x=b$, divided by the mass of the smaller of $x < a$ and $x > b$. Here, each infinitesimal conductance is equal to $s(x)$ and each infinitesimal mass is equal to $m(x)$.}
  \label{fig:boat1}
\end{figure}
\end{mdframed}

The \textbf{Hardy-Muckenhoupt Inequality}~\cite{muckenhoupt1972hardy, MillerHardy18} states that, up to constant factors, the $(m,s)$ fundamental eigenvalue is bounded above by the $(m, s)$ Hardy value of any Hardy split, and bounded below by the infimum $(m, s)$ Hardy value over all Hardy splits.  In our proof, we will only use the upper bound on the fundamental eigenvalue.\footnote{The upper bound is straightforward from simple bounds on the Rayleigh quotient formulation of the fundamental eigenvalue~\cite{muckenhoupt1972hardy}. The difficulty in proving the Hardy-Muckenhoupt inequality is in the lower bound, which we do not use in this proof.}

\subsection{Density Cheeger-Buser in one dimension}
\begin{proof} (of Theorem~\ref{thm:cheeger-buser} in one dimension)

Let $\rho$ be any one-dimensional $L^1$ function $\R \rightarrow \R$. Let $R: \R \rightarrow \R^{\geq 0}$ be a $1$-Lipschitz function, where for all $z$, $\rho(x)/\rho(y)$ is bounded above by $2$ for all $x, y$ in balls of radius $R(z)$ centered at $z$. 
\begin{mdframed}
\begin{figure}[H]
  \includegraphics[width=\linewidth]{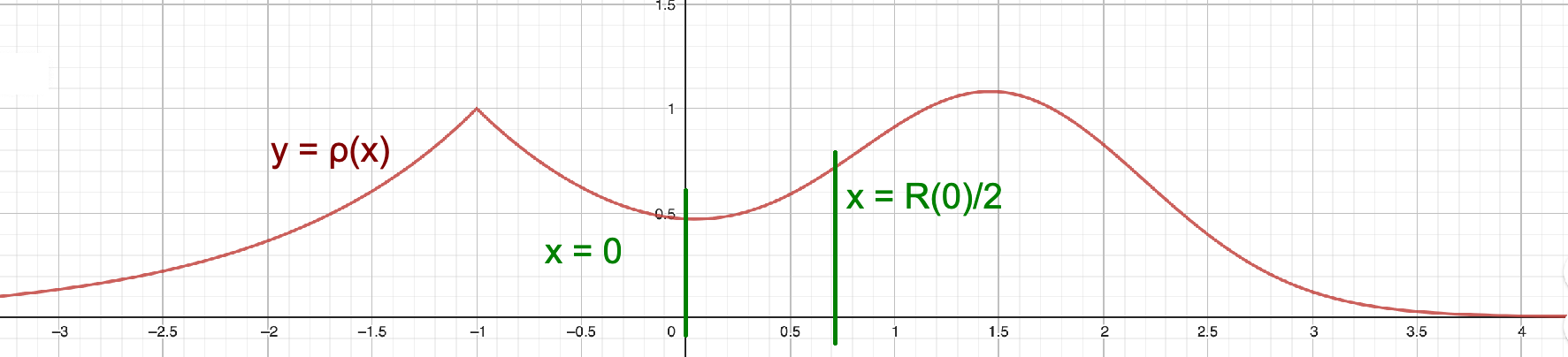}
  \caption{(Proof Sketch) Suppose our $(\rho, \rho R)$ isoperimetric cut is at $ x = 0$.  The Hardy-Muckenhoupt inequality tells us the $(\rho, \rho R^2)$ fundamental eigenvalue is upper bounded by the $(\rho, \rho R^2)$ Hardy value of the split $(0, R(0)/2)$. We will prove this Hardy value is upper bounded by the $(\rho, \rho R)$ isoperimetric constant. $R(0)$ has the property that $\rho$ does not vary by more than a multiple of $2$ on $[-R(0), R(0)]$.
\\
\\ 
The numerator of the $(\rho, \rho R^2)$ Hardy value will be bounded by a constant multiple of $\rho(0) R(0)$. The denominator of this Hardy value will be bounded by a constant multiple of the smaller of $\int_{-\infty}^0 \rho(x)dx$ and $\int_0^{\infty} \rho(x)dx$. 
}
  \label{fig:boat1}
\end{figure}
\end{mdframed}
First, we prove the Buser inequality, by upper bounding the $(\rho, \rho R^2)$ fundamental eigenvalue by a $(\rho, \rho R^2)$ Hardy value of some Hardy split, and then upper bounding that Hardy value by the $(\rho, \rho R)$ isoperimetric constant. In one dimension, any isoperimetric cut can be represented by a single point $x$, where the two sides of the cut are $\{y : y \leq x\}$ and $\{y : y > x\}$. Without loss of generality, suppose the $(\rho, \rho R)$ isoperimetric cut is at $x = 0$.

The $(\rho, \rho R^2)$ fundamental eigenvalue is upper bounded by the $(\rho, \rho R^2)$ Hardy value for any Hardy split, by the Hardy-Muckenhoupt inequality. Without loss of generality, suppose that $\int_{-\infty}^0 \rho(x) \leq \int_0^\infty \rho(x)$. Now consider the Hardy split defined by $(0, R(0)/2)$. The Hardy Value of the split equals:
\begin{align}\label{eq:cond}
    \frac{1/\left(\int_0^{R(0)/2} \frac{dx}{\rho(x) R(x)^2}\right)}{\min \left( \int_{\infty}^0 \rho(x)dx, \int_{R(0)/2}^\infty \rho(x) dx\right)} 
\end{align}

The numerator of the Hardy value of this split is equal to:
\begin{align}\label{eq:cond}
    \frac{1}{\int_0^{R(0)/2} \frac{dx}{\rho(x) R(x)^2}} 
\end{align}
By the definition of $R$, we know that $\rho(x)$ doesn't change by more than a constant factor in the interval $[0, R(0)/2]$. Since $R$ is $1$-Lipschitz, $R(x)$ also doesn't change by more than a constant factor on this interval.  Therefore, the value in Equation~\ref{eq:cond} is a constant multiple of 
\begin{align*}
&\frac{1}{\int_0^{R(0)/2} \frac{dx}{\rho(0) R(0)^2}}
\end{align*}
which equals $\rho(0)R(0)/2$.  Thus, the numerator of the Hardy Value is $\Theta(\rho(0)R(0))$.

Now we bound the denominator of the Hardy Value. We claim that $\int_{R(0)/2}^\infty \rho(x) dx =  \Theta\left(\int_{0}^\infty \rho(x) dx \right)$. This is because $\int_{R(0)/2}^{R(0)} \rho(x)dx = \Theta \left(\int_0^{R(0)} \rho(x)dx\right)$, by the definition of $R(0)$. Putting together the numerator and denominator bound, the $(\rho, \rho R^2)$ Hardy value of the Hardy Split $(0, R(0)/2)$ is 
\[ \Theta\left(\frac {\rho(0)R(0)}{\min\left( \int_{-\infty}^0 \rho(x) dx, \int_0^\infty \rho(x) dx \right) }\right) \]
which is a constant multiple of the isoperimetry of the $(\rho, \rho R)$ cut at $0$, which we assumed without loss of generality was optimal. Therefore, the $(\rho, \rho R^2)$ fundamental eigenvalue of any density is bounded above by a constant multiple of the $(\rho, \rho R)$ cut in one dimension, which proves the Buser inequality in Theorem~\ref{thm:cheeger-buser} for one dimension. The Cheeger inequality in this case is a straightforward derivation from the original proof~\cite{Cheeger70}, and has been deferred to Appendix~\ref{app:cheeger}.

\end{proof} 
\section{Technique for Density Cheeger-Buser in $d$ dimensions}\label{sec:technique}
We prove the Buser inequality in Theorem~\ref{thm:cheeger-buser} by using mollification over balls of radius $R$. Consider an indicator function of the $(\rho, \rho R)$ isoperimetric cut: the indicator function is $1$ on one side of the cut and $0$ on the other. For each point $x$, we create a new function that is the weighted average of our indicator function in a ball of radius $R(x)$ around $x$ (this is also known in the literature as `mollification' on a ball of radius $R(x)$ around $x$ ~\cite{mollifiers44}). We will show the $(\rho, \rho R^2)$ Rayleigh quotient of this new function is bounded above by the weighted isoperimetry of the cut.

Although the idea is simple, the details for proving that this idea works are technical, and have been deferred to Appendix (the entire appendix is a proof of Theorem~\ref{thm:cheeger-buser}). These computational details are \textit{not} standard: past proof techniques for Buser type inequalities generally involve mollification of balls of constant radius (or relied on heat kernel arguments with a constant specific heat capacity~\cite{ledoux2004spectral}), rather than with balls of varying radii. Mollifying with balls of varying radii introduces computational issues from extra terms not present when the radii are all the same~\footnote{For example, the integral of the function after mollification is \textit{not} equal to the integral of the function before mollification when the molliifcation balls have varying radii, in contrast to the case when the radius does not vary.}, and much of our proof is devoted to bounding these extra terms. Despite the technical complications, we have organized our proof in the appendix in a clear and organized way, and we believe it is accessible to the interested reader with limited background on mollification.

This technical argument is similar to that presented in~\cite{cmww20}, but differs in a key way: their work only applied only to the case when $\rho$ was $L$-Lipschitz, and their work can be interpreted as a strict corollary of ours when $R$ is forced to be $\rho/L$ (see Corollary~\ref{cor:lip}). Forcing $R$ in this fashion leads to spectral cuts for \textit{any} mixture of two one-dimensional Gaussians being arbitrarily far out on the tail, which is undesirable. Our work avoids such issue. Compared to~\cite{cmww20}, our proof is cleaner and disentangles $\rho$ and $R$ in a careful way.

We are allowed to have $\rho(x)$ equal to $0$ for some $x$, which differs from nearly every other proof about spectral cuts~\cite{von2008consistency, TrillosConsis16, TrillosContin16, TrillosRate15, TrillosVariational15}. No non-empty ball of radius $R(x)$ around point $x$ contains $y$ where $\rho(y) = 0$. This is key: if $\rho$ has a cut with $0$ isoperimetry, it is necessary for our proof method that any ball of mollification does not `cross' this $0$ cut.

As mentioned earlier, our technique for the Cheeger component of Theorem~\ref{thm:cheeger-buser} is straightforward from existing work, and can be found in Appendix~\ref{app:cheeger}.
\section{Conclusion and Open Questions}\label{sec:conclusion}
We give the first result on isoperimetry of spectral cuts, weighted or otherwise, on general probability densities. This informs us when continuous spectral clustering finds clusters with good isoperimetry, and when it doesn't. We list some future directions of work:

\textbf{One Cheeger-Buser to Rule Them All. }
Is there be a single, unified inequality for probability densities supported on manifolds, that covers all of: the Cheeger-Buser inequality in Theorem~\ref{thm:cheeger-buser}, the tighter Cheeger inequality for log concave densities (see Section~\ref{sec:past-work}), and the manifold Buser inequality~\cite{Buser82}? This setting is relevant to manifold learning~\cite{belkin2004semisup, belkin2007convergence}, which assumes data can be modeled as samples from a density supported on a manifold. 

\textbf{A Simpler Proof of Density Buser using Hardy Inequalities. }
In Section~\ref{sec:1d}, we use the Hardy-Muckenhoupt inequality to show a Buser inequality (Theorem~\ref{thm:cheeger-buser}) on probability densities in one dimension. Can we use the Localization Lemma~\cite{KLS95} to generalize this proof to arbitrary dimension? This would simplify our full proof in Appendix~\ref{app:buser-start} through~\ref{app:buser-end}. Can the localization lemma improve dimension dependence?

\textbf{Convergence of Spectral Clustering to Spectral Cuts. }
Our paper combined with~\cite{TrillosVariational15}, gives intuition that spectral clustering on a large number of samples drawn from a mixture of Laplace distributions, converges to a cut of the mixture with low isoperimetry. However, this intuition is not yet a formal proof.  Currently, the convergence proofs of~\cite{TrillosVariational15} (linking spectral clustering on large datasets to unweighted spectral cuts) require that the density be of bounded support: can we do away with this assumption?

\textbf{New Spectral Clustering Algorithms. }
Our work suggests there exist weighted variants of spectral clustering algorithms on samples from a density, which approximate our weighted spectral cut when the number of samples grows large. Do these algorithms work well in practice? What are these algorithms, and can we prove they converge to weighted spectral cuts on the underlying density? What if we know something about the structure of the underlying density, such as knowing it is the mixture of Gaussians? 

\textbf{Manifold Buser \textit{without} lower bounds on Ricci Curvature. }
Buser's inequality on manifolds~\cite{Buser82} requires a lower bound on the Ricci curvature. This is necessary as long as the isoperimetry and eigenvalue of the manifold are unweighted. If one allows for weighted isoperimetry and weighted eigenvalues, can we find a weighted Buser inequality that applies to all manifolds (possibly by including point-wise Ricci curvature into the weights), for which the manifold Buser inequality is a corollary?

\subsection{Acknowledgements}
We would like to thank Emanuel Milman, Luca Trevisan, and Michel Ledoux for explaining the state of the art in Buser inequalities on probability densities, manifolds, and graphs. We would also like to thank Alex Wang for helpful discussions.

\addcontentsline{toc}{section}{References}
\bibliographystyle{alpha}
\bibliography{main}
\pagebreak
\addcontentsline{toc}{section}{Appendix}
\begin{appendix}
\section*{Appendix}
\section{Density whose Unweighted Spectral Cut has poor Isoperimetry}~\label{app:counterexample}
In this section, we show that for any $K$, there exists a density function whose unweighted isoperimetric constant is $O(1/2^K)$, and whose unweighted spectral cut has unweighted isoperimetry $O(1/K)$.

Consider function $\rho$ supported on $[-1, 1] \times [0, K]$ for constant $K$, and let:

\[
\rho(x, y) = |x| + \frac{1}{2^K}
\]
$\rho$ is constructed to be the Cartesian product of $|x| + \frac{1}{2^K}$ from $[-1, 1]$, and function that is $1$ on $[0, K]$. The unweighted fundamental eigenvalue of $|x| + \frac{1}{2^K}$ on $[-1, -1]$ is $O(1/K)$, by the Hardy-Muckenhoupt inequality (see Section~\ref{sec:1d}). The unweighted fundamental eigenvalue of the function that is $1$ on $[0, K]$ is $O(1/K^2)$.

By symmetry arguments, it can be shown that the $(\rho, \rho)$ fundamental eigenfunction is either constant on the lines $x = C$ for all constants $C$, or else constant on the lines $y = C$ for all constants $C$.  From this, we can deduce that for large $K$, the unweighted spectral cut is defined by $y \leq K/2$, which has unweighted isoperimetry $O(1/K)$. The isoperimetric constant is $O(1/2^K)$, which occurs for the cut defined by $x \leq 0$.  Thus, the unweighted spectral cut can have unweighted isoperimetry that is exponentially worse than optimal. The computation is exactly the same when we scale $\sqrt{\rho}$ to have integral $1$, so that $\sqrt{\rho}$ a true probability density function (continuous spectral clustering on probability density $\sqrt{\rho}$ is the same as unweighted spectral cut on $\rho$, see Section~\ref{sec:past-work-spectral-clustering} and~\cite{TrillosVariational15}).

This work combined with with~\cite{TrillosVariational15} (see Section~\ref{sec:past-work-spectral-clustering}), gives evidence that two-way normalized spectral clustering may give undesirable clusters when applied to $\sqrt{\rho}$.
\section{Spectral Cut Theorems follow from Density Cheeger-Buse}\label{app:spectral}
We prove that Theorem~\ref{thm:cheeger-buser} implies Theorem~\ref{thm:weighted-spectral-cut}. This derivation is straightforward from past work, but we include it here for completeness.

\begin{proof} (Of Theorem~\ref{thm:weighted-spectral-cut})

Let $A$ be the $(\rho, \rho R^2, \rho R)$ weighted spectral cut, where $\rho, R, L, C, d$ are defined as in Theorem~\ref{thm:weighted-spectral-cut}. Then Lemma~\ref{lem:cheeger-strong} and the Buser inequality in Theorem~\ref{thm:cheeger-buser} imply:

\[ \Phi^2_{m,c} \leq \Phi^2_{m, c}(A) \leq O(\lambda_2^{m, s}) \leq O(d(L+1)(\Phi_{\rho, \rho R}) + d\Phi^2_{\rho, \rho R})\]
and Theorem~\ref{thm:weighted-spectral-cut} follows.
\end{proof}
\section{Proof of Cheeger Inequality on Densities}\label{app:cheeger}
In this section, we prove the Cheeger inequality portion of Theorem~\ref{thm:cheeger-buser}. This inequality is straightforward from existing work like~\cite{Cheeger70, AlonM84}, but we include a derivation here for completeness.

\begin{thm}
Let $m, c, s: \R^d \rightarrow \R_{\geq 0}$. Then:

\[ 
\Phi{m, c}^2 \leq 4 \| c / \sqrt{ms} \| _\infty^2 \lambda_2^{m, s}
\]
\end{thm}
\begin{proof}
Let $g: \R^d \rightarrow \R_{\geq 0}$ be a differentiable function, and let $g = u^2$. By Cauchy Schwarz:

\begin{align*}
&\int_{\R^d} c(x) |\nabla g(x)| dx 
\\
&= 2 \int_{\R^d} c(x) |u(x)||\nabla u(x)| dx 
\\
&\leq  2 \sqrt{\int_{\R^d} c(x)^2/m(x) |\nabla u(x)|^2 dx } \sqrt{\int_{\R^d} m(x) u(x)^2 dx} 
\\
& \leq 2 \| c/\sqrt{ms} \|_\infty \sqrt{\int_{\R^d} s(x) |\nabla u(x)|^2 dx} \sqrt{\int_{\R^d} m(x) u(x)^2 dx}
\end{align*}

Dividing through by $\int_{\R^d} m(x) g(x) dx $, setting $g$ as members of the sequence of differentiable functions that converge to the indicator function given by the isoperimetric cut, and squaring both sides gives us our desired result.

\end{proof}
When we set $ m = \rho, c = \rho R, s = \rho R^2$ for any $\rho, R : \R^d \rightarrow \R_{\geq 0}$, the Cheeger inequality of Theorem~\ref{thm:cheeger-buser} follows. We can prove a slightly stronger theorem with the same proof, which will be useful in proving Theorem~\ref{thm:weighted-spectral-cut} on weighted spectral cut isoperimetry.

\begin{lem}\label{lem:cheeger-strong}
Let $A$ be the $(m, s, c)$ spectral sweep cut. Then:
\[
\Phi_{m, c}(A)^2 \leq 4\lambda_2^{m, s}
\]
\end{lem}
\section{Proof of Buser Inequality on Densities}\label{app:roadmap}\label{app:buser-start}
The remainder of the appendix is devoted to the technical proof of the Buser inequality in Theorem~\ref{thm:cheeger-buser}. The Cheeger portion of the inequality is proven in Appendix~\ref{app:cheeger}.

\begin{itemize}
    \item Section~\ref{sec:notation} contains notation and conventions that we use throughout the Appendix.
    \item Section~\ref{sec:rayleigh} introduces two Lemmas to bound the numerator and denominator of a key Rayleigh quotient. We will show how to prove Buser's inequality from these two lemmas. The numerator lemma will be proven in Section~\ref{sec:num-proof}, and the denominator lemma will be proven in Section~\ref{sec:denom-proof}.
    \item Section~\ref{sec:tools} gives two tools we use to prove the numerator and denominator bounds.
    \item Section~\ref{sec:num-proof} proves the numerator bound Lemma from Section~\ref{sec:rayleigh}.
    \item Section~\ref{sec:denom-proof} proves the denominator bound Lemma from Section~\ref{sec:rayleigh}.
\end{itemize}

\subsection{Notation}\label{sec:notation}
In this section, we list notation to simplify our calculations.
\begin{itemize}
\item $u_{\theta}(x) := \int_{y \in \R^d} u(x - \theta R(x) y) \phi(y) dy$ for some function $R:\R^d \rightarrow R$ and some mollifier $\phi: \R^d \rightarrow \R^{\geq 0}$ supported on the unit ball centered at $0$, which we notate as $B(0, 1)$. Here, we force $\int \phi(y) = 1$. In our write-up, $\theta$ is always a constant. 
\begin{itemize}
    \item This is equivalent to taking the ball of radius $\theta R(x)$ around x, and finding the weighted average of $u$
    in that ball.
    \item $u_\theta$ depends on both $R$ and $\phi$.
\end{itemize}
\item $\overline{u}$ denotes the constant $\left. \int_{\R^d} \rho(y) u(y) dy \middle/ \int_{\R^d} \rho(y) dy \right.$, which is equal to the $\rho$-weighted average of $u$.
\item  $a \lesssim b$ will mean $a <= Cb$ for some constant $C$.

 \item $a \approx b$ will mean $cb \leq a \leq Cb$ for some positive constants $c$ and $C$.
\item  $\grad_x(f(x, y))$ means that $f(x, y)$ will be treated as a function in $x$ with constant $y$, and the gradient will be taken accordingly.  
 
\item  $\gu (f(x))$ means that the gradient of $u$ will be applied to $f(x)$. This is in contrast with $\grad(u(f(x))$, where the gradient is applied to the function $u \circ f$

\item The notation $\rho R (f(x))$ is equivalent to $\rho(f(x)) R(f(x))$. This shorthand will be convenient in simplifying certain equations.
\item $\|u\|_q = \left(\int\|u(y)\|^q dy\right)^{1/q} $
\item $R(x)$ in this paper will generally have the property that $\rho$ doesn't vary by more than a factor of $2$ on the ball of radius $R(x)$ centered at $x$.
\item The notation $(\rho R |\grad u|_2)_{\theta}$ is equivalent to defining $f(x) := \rho(x) R(x) |\gu (x)|_2$ and computing $f_\theta(x)$.
 \item $J_x(f)(x')$ refers to the Jacobian of function $f: \R^{d_1} \rightarrow \R^{d_2}$, as a function of variable $x \in \R^d$, applied at point $x' \in \R^d$.
\end{itemize}
No serious attempts were made to optimize the non-dimensionally sensitive constants.
\subsection{Rayleigh Quotient Bound}\label{sec:rayleigh}

\begin{lem} \label{lem:overall} (Buser Inequality) 

For all indicator functions $u$ that is $1$ on $A \subset \R^d$ and $0$ elsewhere, $\theta < \min(1, \frac{1}{2L})$, $\rho$, and $L$-Lipschitz $R$ where $\rho(y) R(y)$ varies by at most a factor of $4$ for y in ball of radius $R(x)$ around x: 
\begin{align}\label{eq:overall}
\lambda_2 \leq \frac{\|\rho R^2 \| \gut \|_2^2\|_1}
{\|\rho (\ut - \overline{\ut})^2\|_1} \lesssim \frac{( dL + d / \theta )\Phi_u }{\frac{1}{4} - \theta \Phi_u}
\end{align}
where 
\[ \Phi_u :=  \frac{\| {\rho |\gu|_1}\|_1}{\|\rho (u - \overline{u})\|_1}
\approx
\frac{\partial_{\rho R} (A)}{\min(Vol_\rho(A), Vol_\rho(A^c))} \]
Setting $\theta = \min(1, 1/(2L), 1/(8\Phi_u))$ and set $u$ to be the indicator function of $A$ corresponding to the optimal $(\rho,\rho R)$ cut:
\[
\lambda_2 \leq O(\max(\Phi (3dL + d), \Phi^2))
\]
\end{lem}
To prove this, we bound the numerator and denominator of the Rayleigh quotient in Equation~\ref{eq:overall}. 
\begin{lem}\label{lem:num} (\textbf{Numerator bound}) 

For $\rho$ and $R$ where $R$ is $L$-Lipschitz, $\rho R$ doesn't change by more than a factor of $4$ on a ball of radius $R$, and any $\theta$ where $\theta L \leq 1/2$. 
\[ 
\| \rho R^2 |\gut|_2^2 \|_1 \leq \frac{6d(2+3L)}{\theta} \|u\|_\infty \| \rho R | \gu |_2 \|_1
\]
\end{lem}
\begin{lem} \label{lem:denom} (\textbf{Denominator Bound})

For $\rho$ and $R$ where $R$ is $L$-Lipschitz, $\rho R$ doesn't change by more than a factor of $4$ on a ball of radius $R$, and any $\theta$ where $\theta L \leq 1/2$. 
\[
\| \rho (\ut - \overline{\ut})^2\|_1  \geq \frac{1}{4} \|\rho(u - \overline{u})\|_1 - 2\theta \|\rho R \gu \|_1
\]
\end{lem}

\begin{proof} (of Lemma~\ref{lem:overall}) Lemma~\ref{lem:num} and~\ref{lem:denom} tell us that, for indicator functions $u$ (guaranteeing $\|u\|_\infty \leq 1$):

\[
 \frac{\|\rho R^2 \| \gut \|_2^2\|_1}
{\|\rho (\ut - \overline{\ut})^2\|_1} \leq \frac{6d(2+3L) \|\rho R |\gu|_2\|_1 }{\theta \left(\frac{1}{4}\|\rho(u - \overline{u})\|_1 - \theta \|\rho R |\gu|_2\|_1\right)}
\]
Divide the numerator and denominator by $\|\rho(u - \overline{u})\|_1$, and Lemma~\ref{lem:overall} follows.
\end{proof}
Now we prove Lemma~\ref{lem:num} and~\ref{lem:denom}.    

\vspace{3 mm}
\subsubsection{Numerator Bound Proof}
We prove the numerator bound in Lemma~\ref{lem:num} by introducing two lemmas, whose proofs we defer to Section~\ref{sec:num-proof}:

\begin{lem}\label{lem:num-inf}

Given any $L$-Lipschitz function $R : \R^d \rightarrow \R^{\geq 0}$ and any function $u$,
\[ 
\| R(x) |\gut(x)|_2 \|_{\infty} \leq \left(\frac{O(d)}{\theta} + 2dL\right) \|u\|_\infty \]
\end{lem}
\begin{lem}\label{lem:num-one}
For $\rho$ and $L$-Lipschitz $R$ where $\rho R$ doesn't change by more than a factor of $4$ on a ball of radius $R$, and any $\theta$ where $\theta L \leq 1/2$:
\[ 
\| \rho R |\gut|_2 \|_1 \leq 6 \| \rho R | \gu|_2 \|_1 
\] 
\end{lem}
\begin{proof} (of Lemma~\ref{lem:num}, numerator bound)
We know for all $a, b$ that $\|ab\|_1 \leq \|a\|_\infty \|b\|_1$. Set $a := R |\gut|_2$ and $b := \rho R |\gut|_2$, and apply Lemmas~\ref{lem:num-inf} and~\ref{lem:num-one}.
\end{proof}

\subsubsection{Denominator Bound Proof}

Now we prove the denominator bound in Lemma~\ref{lem:denom}, by introducing two lemmas. The proof of the latter is long, and deferred to Section~\ref{sec:num-proof}. 
\begin{lem} \label{lem:denom-one}
For all $\rho$ and $R$ (no restrictions) and indicator functions $u$
\[
\| \rho (\ut - \overline{\ut})^2 \|_1 
\geq
\frac{1}{4}\|\rho(u-\overline{u}) \|_1 - \| \rho(\ut - u) \|_1 
\]
\end{lem}
\begin{proof} (of Lemma~\ref{lem:denom-one})
We know by triangle inequality:

\[ \| \sqrt{\rho} (\ut - \overline{\ut}) \|_2 \geq \| 
\sqrt{\rho} (u - \overline{u}) \|_2 - \| \sqrt{\rho} \left( 
(\ut - \overline{\ut}) - (u - \overline{u}) \right) \|_2
\]
If $a \geq b - c$, then $a^2 \geq b^2 - 2bc + c^2 \geq b^2/2 - c^2$.  Thus:
\begin{align} \label{eq:ineq0}
\| \sqrt{\rho} (\ut - \overline{\ut}) \|^2_2 \geq \frac{1}{2}\| 
\sqrt{\rho} (u - \overline{u}) \|^2_2 - \| \sqrt{\rho} \left( 
(\ut - \overline{\ut}) - (u - \overline{u}) \right) \|^2_2
\end{align}

First, note that all functions $f$, $ \int \rho (f - \overline{f})^2 \leq \int \rho f^2$. This follows from basic calculus: this can be seen by replacing $\overline{f}$ in the previous computation with a constant $c$ and optimizing for $c$. In particular, this holds if $f = (\ut - u)$ and $\hat{f} = \overline{\ut - u}$   Second, note that $\|\ut - u\|_{\infty} \leq 1$ when $u$ is an indicator function. Therefore:
\begin{align}\label{eq:ineq1}
 \| \sqrt{\rho} \left( 
(\ut - \overline{\ut}) - (u - \overline{u}) \right) \|^2_2 \leq
\|\sqrt{\rho} (\ut - u) \|^2_2  \leq \|\rho (\ut - u) \|_1 \|\ut- u\|_\infty \leq \|\rho (\ut - u) \|_1 
\end{align}

Now let $A$ be the set that $u$ is an indicator function for. Note that $\int_A \rho(u - \overline{u}) = -\int_{A^C} \rho(u - \overline{u})$, which follows since $\int_A \rho(u - \overline{u}) = 0$ by the definition of $\overline{u}$. We note that on $A, u \geq \overline{u}$ and on $A^C, u \leq \overline{u}$, and thus $\int_A \rho(u - \overline{u}) = -\int_{A^C} \rho(u - \overline{u}) = \frac{1}{2} \int |\rho(u - \overline{u})|$. Now notice that $\int \rho (u - \overline{u})^2 = (1 - \overline{u})\int_A \rho(u - \overline{u}) - \overline{u} \int_{A^C} \rho (u - \overline{u}))   = \frac{1}{2} \int |\rho (u - \overline{u})|$. Thus:

\begin{align}\label{eq:ineq2}
 \frac{1}{2} \| \sqrt{\rho} (u - \overline{u})\|_2^2 = \frac{1}{4}\| \rho (u - \overline{u}\|_1 
 \end{align}

Combining Inequalities~\ref{eq:ineq0},~\ref{eq:ineq1}, and~\ref{eq:ineq2} completes the proof of our lemma.
\end{proof}

\begin{lem} \label{lem:denom-two}
For $\rho$ and $R$ where $R$ is $L$-Lipschitz, $\rho R$ doesn't change by more than a factor of $4$ on a ball of radius $R$, and any $\theta$ where $\theta L \leq 1/2$, 
\[ 
\| \rho(\ut - u) \|_1 \leq 2\theta \| \rho R \gu \|_1
\]
\end{lem}

\begin{proof}(of Lemma~\ref{lem:denom}, denominator bound)
This follows immediately from Lemma~\ref{lem:denom-one} and~\ref{lem:denom-two}.
\end{proof}

Therefore, we have proven the Buser inequality in Lemma~\ref{lem:overall}, conditional on proofs for Lemma~\ref{lem:num-inf},~\ref{lem:num-one},~\ref{lem:denom-one}, and~\ref{lem:denom-two}. See Section~\ref{sec:denom-proof} for these proofs.

\subsection{Useful Tools}\label{sec:tools}
\begin{lem}\label{lem:moll} (\textbf{Bounding the $L_1$ norm of a mollified function})
Let $R$ be any $L$-Lipschitz function, and let $f: \R^d \rightarrow \R^{\geq 0}$ be a non-negative function. Then $\frac{1}{1+\theta L}\|f\|_1 \leq \|f_{\theta}\|_1 \leq \frac{1}{1-\theta L}\|f\|_1$ for any $f$ when $\theta L < 1$.
\end{lem}
We note that in the case of constant $R$, the inequalities become equality. 
We defer the proof to Appendix~\ref{app:tool-proof}
\begin{lem} \label{lem:moll-bound}
There exists a smooth mollifier $\phi(y)$ supported on the unit ball where $\int \phi(y) = 1$ and $\int_{y \in \R^d} |\grad \phi(y)|_2 dy  \leq 2d$.
\end{lem}
\begin{proof}
This fact about mollifiers is known in the folklore, and one such mollifier can be constructed by letting $\phi$ be smooth approximations of the indicator function on the unit ball, scaled so that the integral of $\phi$ is $1$. 
\end{proof}

\begin{lem} \label{lem:weighted-moll-int}
For any differentiable radially symmetric $\phi:\R^d \rightarrow \R^{\geq 0}$ satisfying $\int \phi(y) = 1$ and $\phi$ decreasing radially outwards:
\[ 
    \left|\int_{\R^d} f(y) \grad \phi(y)dy \right|  \leq O(d)\|f\|_\infty
\] 
for any $f : \R^d \rightarrow \R$.
\end{lem}
\begin{proof}
This follows directly from Lemma~\ref{lem:moll-bound}.
\end{proof}
\subsection{Proofs for Numerator Bound Lemmas} \label{sec:num-proof}
\begin{proof} (Of Lemma~\ref{lem:num-inf})
Recall the multivariable integration by parts formula:
\[ 
\int_V (\grad f) g = \int_V \left(\grad(fg) - f \grad(g)\right) = \oint_S fg\hat{n} - \int_V f \grad(g)
\]
where $S$ is the surface of $V$ and  $\hat{n}$ is the outward unit normal of $S$. The key to our proof is use of this formula. (Proof continued on next page).

\begin{align*} 
R&(x) |\grad \ut (x)|_2   = R(x) \left| \grad_x \left( \int_{y \in B(0,1)} u(x - \theta R(x) y) \phi(y) dy \right) \right|_2 \text{\qquad (def. of $u_\theta$)}
\\
& = R(x) \left|\int_{y \in B(0,1)} \grad_x (u(x-\theta R(x) y))\phi(y) dy\right|_2  \text{\qquad (linearity of integral)}
\\
& = R(x) \left|\int_y (I - \theta \grad R(x) \otimes y)\cdot \grad u(x - \theta R(x) y) \phi(y)\right|_2  dy
\\ \nonumber
& \text{\qquad (Chain rule, using $J_x(x - \theta R(x)y) = I - \theta \grad R(x) \otimes y)$}
\\ 
& = R(x) \left|\int_y (I - \theta \grad R(x) \otimes y)\cdot  \frac{\theta R(x) \grad u(x - \theta R(x) y)}{\theta R(x)} \phi(y)\right|_2  dy
\\
& = R(x) \left|\int_y (I - \theta \grad R(x) \otimes y)\cdot \frac{-\grad_y (u(x - \theta R(x) y))}{\theta R(x)} \phi(y)\right|_2  dy
\\
\nonumber 
& \text{\qquad \qquad (from chain rule: $\grad_y(u(x-\theta R(x) y) = -\theta R(x) \grad u(x-\theta R(x)y)$)}
\\
& = \left|\int_y (I - \theta \grad R(x) \otimes y)\cdot \frac{-\grad_y u(x - \theta R(x) y)}{\theta} \phi(y)\right|_2  dy
\\
& = \left|\int_y (I - \theta \grad R(x) \otimes y)\cdot \frac{u(x - \theta R(x) y)}{\theta} \grad \phi(y)\right|_2
\\
& \nonumber \text{\qquad \qquad (Integrate by parts, since $\phi(y)$ vanishes on the surface of $B(0, 1)$)}
\\
& = \left|\frac{\int_y u(x - \theta R(x)) \grad \phi( y) dy}{\theta}  - \int_y (\grad R(x) \otimes y) \cdot u(x - \theta R(x)  y) \grad \phi(y) dy \right|_2
\\
& \leq \frac{O(d)\|u\|_\infty}{\theta} + \left|\int_y (\grad R(x) \otimes y) \cdot u(x - \theta R(x)  y) \grad \phi(y) dy \right|_2 \text{\qquad (By Lemma~\ref{lem:weighted-moll-int})}
\\
& \leq \frac{O(d)\|u\|_\infty}{\theta} + L \int_y \left|u(x-\theta R(x) y )| |\grad \phi(y)|_2 dy \right|_2
\\ \nonumber 
& \text{\qquad (When $\grad \phi(y) \not= 0,$ then $|y|_2 \leq 1$ and operator norm of $\grad R(x) \otimes y \leq L$)}
\\
& \leq \left(\frac{O(d)}{\theta} + 2dL\right)\|u\|_\infty \text{\qquad (By Lemma~\ref{lem:moll-bound}, bounding $\int_y |\grad \phi(y)|_2$ with $2d$).}
\end{align*}
\end{proof}

This completes our proof of Lemma~\ref{lem:num-inf}. Now we prove Lemma~\ref{lem:num-one} via a series of claims:
\begin{claim}\label{claim:theta-comm}

\[ |\gut(x)|_2 \lesssim (|\gu|_2)_\theta(x)\]
(for $L$-Lipschitz function $R$, for all $u$ and all $x$). 
\end{claim}
\begin{claim}\label{claim:num-rhoR}
\[ (\rho R |\gut|_2)(x)\lesssim (\rho R|\gu|_2)_\theta(x) \]
(for all x, where we have restrictions on $\rho R$ and $\theta L$.) 
\end{claim}
\begin{claim} \label{claim:num-one}
\[ 
\|(\rho R|\gu|_2)_\theta \|_1 \lesssim \| \rho R |\gu|_2\|_1
\] (for all $R$ with $L$-Lipschitz $R$ and $\theta L \leq \frac{1}{2}$. No assumptions on $\rho$).
\end{claim}

First, we show how to use these claims to prove our numerator bound. Then we provide full proofs of these claims.
\begin{proof} (of Lemma~\ref{lem:num-one}, numerator bound)
Combine Claim~\ref{claim:num-rhoR} and Claim~\ref{claim:num-one}.\footnote{Note that the weight we attached to $\gut$ was $\rho R$ in these claims, but we could have used any monomial in $\rho$ and $R$ (at the cost of some constant factors)}.

\end{proof}

\begin{proof} (of Claim~\ref{claim:theta-comm}, used in proving Claim~\ref{claim:num-rhoR})
\begin{align*}
    & |\gut(x)|_2 
    \\ &= \left|\grad_x \int_y u(x - \theta R(x) y)\phi(y) dy\right|_2 
    \text{\qquad (Definition of $u_\theta$)}
    \\
    &= \left|\int_y \grad_x(u(x-\theta R(x)y)\phi(y) dy\right|_2 \text{\qquad (Linearity of gradient)}
    \\
    &= \left|\int_y (I - \theta \grad R(x) \otimes y) \cdot \gu(x - \theta R(x)y)\phi(y) dy \right|_2 
    \text{\qquad (Multivariable chain rule)}
    \\
    &\leq (1 + \theta L) \int_y \left|\gu(x - \theta R(x) y)  \right|_2 \phi(y)
    \\
    & \text{\qquad \qquad ($R$ is $L$-Lipschitz, and $|y|_2 \leq 1$ when $\phi(y) \not= 0$)}
    \\
    &= (1 + \theta L)  (|\gu|_2)_{\theta}
\end{align*}
\end{proof}
\begin{proof} (of Claim~\ref{claim:num-rhoR})
\begin{align*}
&(\rho R |\gut|_2)(x) 
\\ 
& \lesssim \rho R(x) (|\gu|_2)_\theta(x) \text{\qquad (by Claim~\ref{claim:theta-comm})}
\\ 
&= \int_y \rho R (x)|(\gu)(x - \theta R(x)y)|_2\phi(y) dy
\\
&\leq \max_{y \in B(0,1)}\left(\frac{\rho R(x)}{\rho R(x - \theta R y)} \right)
\int_y \rho R(x-\theta R y)  |(\gu)(x - \theta R(x)y)|_2\phi(y) dy
\\
&= \max_{y \in B(0,1)}\left(\frac{\rho R(x)}{\rho R(x - \theta R y)}\right)  (\rho R |\gu|_2)_\theta (x)
\\
&\leq \frac{2}{1-\theta L} (\rho R |\gu|_2)_\theta (x)
\end{align*}
\end{proof}

\begin{proof}
This is a direct consequence of Lemma~\ref{lem:moll} applied to the function $\rho R | \gu |_2$
\end{proof}
\subsection{Proofs for Denominator Bound Lemmas}\label{sec:denom-proof}
We prove Lemma~\ref{lem:denom-two}, which we've shown suffices to prove Lemma~\ref{lem:denom}. We do so by introducing an intermediate claim.
\begin{claim} \label{claim:denom-utu}
For $L$-Lipschitz $R$ and $\rho(x)$ varying by no more than a factor of two in balls of radius $R(x)$ centered at $R$, given $\theta L \leq \frac{1}{2}$:
\[ 
\rho(x)(\ut(x)-u(x)) \leq 4 \theta \int_0^1 (\rho R |\gu|_2)_{t \theta} (x) dt
\]
\end{claim}
\begin{proof}
\begin{align*} 
& \rho(x)(u_\theta(x) - u(x))
\\
&= \rho \int_y (u(x - \theta R(x) y) - u(x))\phi(y) dy 
\text{\qquad (definition of $u_\theta$)}
\\
& = \theta \rho \int_y \int_{t \in [0,1]}- R(x)y \cdot \grad u(x - t \theta R(x) y) \phi(y) dt dy
\\
&\text{\qquad \qquad \qquad (Chain rule and fundamental theorem of calculus)}
\\
&= \theta  \rho \int_{t \in [0, 1]} \int_y - R(x) y \cdot \gu(x-t\theta R y)\phi(y) dy dt
\text{\qquad (Swap order of integration)}
\\
&= \theta \rho \int_{t \in [0,1]}\int_y - \frac{ R(x)}{ \rho R (x - t \theta R y)}  \rho R (x -  t\theta R y) \left(y \cdot \gu(x-t \theta R y)\right)\phi(y) dy dt
\\ 
&\leq \theta \max_{y \in B(0,1), t \in [0,1]} \frac{ \rho R (x)}{\rho R(x-t\theta R y)} \left|\int_{t \in [0,1]} \int_y  \rho R (x-t\theta R y )(y \cdot \gu(x - t \theta R y))\phi(y)  dy dt \right|
\\
& \leq \theta \frac{2}{1-\theta L} \left|\int_{t \in [0,1]} \int_y  \rho R  (x-t\theta R y )(y \cdot \gu(x - t \theta R y))\phi(y)  dy dt \right|
\\
& \text{\qquad \qquad \qquad (variation in $\rho$ is bounded in balls of radius $R$, and $R$ is $L$-Lipschitz)}
\\
& \leq 4 \theta \left|\int_{t \in [0,1]} \int_y  \rho R  (x-t\theta R y )(y \cdot \gu(x - t \theta R y))\phi(y)  dy dt \right| \text{\qquad (since $\theta L \leq 1/2$)}
\\
& \leq 4 \theta \int_{t \in [0, 1]} \int_y  \rho R (x - t\theta R y) |\gu(x- t \theta R y) |_2 \phi(y) dy dt \text{\qquad (since $|y|_2 \leq 1$ )}
\\
&  = 4 \theta \int_{t \in [0, 1]} (\rho R |\gu|_2)_{t\theta}(x) dt \text{\qquad (Definition of $f_{t\theta}$ when $f(x) := \rho(x) R(x) |\gu(x)|_2$)}
\end{align*}
\end{proof}

\begin{proof} (of Lemma~\ref{lem:denom-two}, key Lemma for proving the denominator bound)
\begin{align*}
& \| \rho(\ut - u) \|_1
\\
& \leq \left\|4 \theta \int_0^1 (\rho R |\gu|_2)_{t\theta}  \right\|_1 dt \text{\qquad (by Claim~\ref{claim:denom-utu})}
\\
& = 4 \theta \int_0^1 \left\|(\rho R |\gu|_2)_{t\theta} \right\|_1 dt
\\
& \leq 4 \theta \left(\int_0^1 \frac{1}{1-t \theta L} dt \right) \|\rho R|\gu|_2\|_1 \text{\qquad (by Lemma~\ref{lem:moll})}
\\
& 
\leq 4 \theta \frac{1}{1-\theta L}\|\rho R|\gu|_2\|_1
\\
& \leq 8 \theta \|\rho R|\gu|_2\|_1  \text{\qquad (since $\theta L \leq 1/2$)}
\end{align*}
as desired.
\end{proof}

\subsection{Proofs for Useful Tools} \label{app:tool-proof}\label{app:buser-end}
\begin{proof} (of Lemma~\ref{lem:moll}) 

The change of variables formula in multivariable calculus, combined with the inverse function theorem, says that for invertible, surjective functions $g : \R^d \rightarrow \R^d$:
\begin{align}
    \int_{x \in \R^d} f(g(x)) dx = \int_{z \in \R^d} \frac{f(z)}{\det (J_x(g)(g^{-1}(z)))} dz
\end{align}
Now let $g(x) := g_y(x) := x - \theta R(x) y$ where $y$ is a constant. 
Note that the Jacobian $J_x(g)(x) = I - \theta \grad R(x) \otimes y$, and the determinant (by the Matrix determinant lemma) is $1 - y \cdot \theta \grad R(x)$.  We note that $g$ is invertible due to the Jacobian determinant being non-zero when $\theta L < 1$, and it is not difficult to see $g$ must be surjective.

Set $z = x - \theta R(x) y$.
\begin{align}
    & \|f_\theta\|_1 
\\
& = \int_{x \in \R^d} \int_{y \in \R^d} f(x - \theta R(x) y) \phi(y) dx dy
\\
& = \int_{y \in \R^d} \phi(y) \int_{x \in \R^d} f(x - \theta R(x) y) dx dy
\\
& = \int_{y \in \R^d} \phi(y) \int_{z \in \R^d} \frac{f(z)}{1 - y \cdot \theta \grad R(g_y^{-1}(z))}
\\
& \text{\qquad (Multivariable Change of variables and inverse function theorem)}
\\ 
& = \int_{z \in \R^d} f(z) \int_{y \in \R^d} \frac{\phi(y)}{1 - y \cdot \theta \grad R(g_y^{-1}(z))}
\end{align}
which is bounded above and below by $\frac{1}{1+\theta L} \|f\|_1$ and $\frac{1}{1 - \theta L} \|f\|_1$, since:
\begin{itemize}
    \item $\int_{y \in \R^d} \phi(y) = 1$
    \item $|y|_2 \leq 1$ when $\phi(y) \not= 0$, and $|\grad R(x)| \leq L$
    \item $\theta L < 1$ 
\end{itemize} by our initial assumptions on $\phi, R, L,$ and $\theta$.
\end{proof}
\end{appendix}

\section{Special Cases of Weighted Spectral Cuts}\label{app:more-cors}
\begin{cor} \label{cor:lip} If $\rho: \R^d \rightarrow \R$ is a Lipschitz function, then
\[ \Omega(\Phi^2_{\rho, \rho^2}) \leq \lambda_2^{\rho, \rho^3} \leq O(d\Phi_{\rho, \rho^2} + \Phi_{\rho, \rho^2}^2) \]
\end{cor}
Corollary~\ref{cor:lip} is the main result of~\cite{cmww20} (Theorem 1.5 in that paper). But, it is a simple consequence of our Theorem~\ref{thm:cheeger-buser}, implying our new theorem is far more general. 

\begin{cor} \label{cor:square-lip} If $\rho$ is the square of a Lipschitz function, then
\[ \Omega(\Phi^2_{\rho, \rho^{3/2}}) \leq \lambda_2^{\rho, \rho^2} \leq O(d\Phi_{\rho, \rho^{3/2} } + \Phi_{\rho, \rho^{3/2}}^2) \]
\end{cor}

The eigenvalue $\lambda_2^{\rho, \rho^2}$ (in Corollary~\ref{cor:square-lip}) appears when studying unnormalized spectral clustering, when a large number of datapoints are drawn from a density $\rho$~\cite{von2008consistency, TrillosVariational15, TrillosRate15}. This corollary shows that the $(\rho, \rho^2, \rho^{3/2})$ spectral cut of the probability density has good $(\rho, \rho^{3/2})$ isoperimetry. This is almost, but not quite the same, as the spectral cut corresponding to unnormalized spectral clustering (which does a $(\rho, \rho^2, \rho^2)$ cut)~\cite{TrillosVariational15}. This theorem suggests that unnormalized spectral clustering is theoretically un-sound, and a new clustering algorithm can be built which converges to the $(\rho, \rho^2, \rho^{3/2})$ cut of the density. 

\begin{cor}\label{cor:gauss} (Weighted Cheeger-Buser for Mixtures of Gaussians)
Let $\rho$ be the mixture of a finite number of Gaussians in $d$ dimensions with variance bounded away from 0. 

Let $R(x) = \max(1, 1/\sqrt{-\log \rho(x)})$. Then:
\[ \Omega(\Phi^2_{\rho, \rho R}) \leq \lambda_2^{\rho, \rho R^2} \leq O(d\Phi_{\rho, \rho R} + \Phi_{\rho, \rho R}^2) \]
\end{cor} 
We note that Corollary~\ref{cor:gauss} is not true if we instead set $R(x) := C$ for any constant $C$. A counterexample can be found by taking two unit variance Gaussians in one dimension, and spacing them arbitrarily far apart. This corollary suggests that a variant of spectral clustering may be more effective on mixtures of Gaussians, compared to standard methods in use today.

\end{document}